\documentclass[sigconf]{acmart}

\AtBeginDocument{%
  }
    
\usepackage{balance}
\usepackage{multirow}
\usepackage{soul}
\usepackage{graphicx}
\usepackage{amsmath}
\usepackage{amsthm}
\usepackage{algorithm}
\usepackage{algorithmic}
\usepackage{bm} 
\copyrightyear{2024}
\acmYear{2024}
\setcopyright{acmlicensed}
\acmConference[MM '24]{Proceedings of the 32nd ACM International Conference on Multimedia}{October 28-November 1, 2024}{Melbourne, VIC, Australia}
\acmBooktitle{Proceedings of the 32nd ACM International Conference on Multimedia (MM '24), October 28-November 1, 2024, Melbourne, VIC, Australia}
\acmDOI{10.1145/3664647.3680620}
\acmISBN{979-8-4007-0686-8/24/10}

\begin{document}

%%
%% The "title" command has an optional parameter,
%% allowing the author to define a "short title" to be used in page headers.
\title{Towards High-performance Spiking Transformers from ANN to SNN Conversion}

%%
%% The "author" command and its associated commands are used to define
%% the authors and their affiliations.
%% Of note is the shared affiliation of the first two authors, and the
%% "authornote" and "authornotemark" commands
%% used to denote shared contribution to the research.
\author{Zihan Huang}
\affiliation{
  \institution{Peking University}
  \city{Beijing}
  \country{China}}
\email{hzh@stu.pku.edu.cn}

\author{Xinyu Shi}
\affiliation{%
  \institution{Peking University}
  \city{Beijing}
  \country{China}}
\email{xyshi@pku.edu.cn}

\author{Zecheng Hao}
\affiliation{
  \institution{Peking University}
  \city{Beijing}
  \country{China}}
\email{1900012989@pku.edu.cn}

\author{Tong Bu}
\affiliation{%
  \institution{Peking University}
  \city{Beijing}
  \country{China}}
\email{putong30@pku.edu.cn}

\author{Jianhao Ding}
\affiliation{%
  \institution{Peking University}
  \city{Beijing}
  \country{China}}
\email{djh01998@stu.pku.edu.cn}

\author{Zhaofei Yu}
\authornote{Corresponding author.}
\affiliation{%
  \institution{Peking University}
  \city{Beijing}
  \country{China}}
\email{yuzf12@pku.edu.cn}
% \thanks{Corresponding author: Zhaofei Yu}

\author{Tiejun Huang}
\affiliation{%
  \institution{Peking University}
  \city{Beijing}
  \country{China}}
\email{tjhuang@pku.edu.cn}

%%
%% By default, the full list of authors will be used in the page
%% headers. Often, this list is too long, and will overlap
%% other information printed in the page headers. This command allows
%% the author to define a more concise list
%% of authors' names for this purpose.
\renewcommand{\shortauthors}{Zihan Huang et al.}

%%
%% The abstract is a short summary of the work to be presented in the
%% article.
\begin{abstract}
Spiking neural networks (SNNs) show great potential due to their energy efficiency, fast processing capabilities, and robustness. There are two main approaches to constructing SNNs. Direct training methods require much memory, while conversion methods offer a simpler and more efficient option. However, current conversion methods mainly focus on converting convolutional neural networks (CNNs) to SNNs. Converting Transformers to SNN is challenging because of the presence of non-linear modules. In this paper, we propose an Expectation Compensation Module to preserve the accuracy of the conversion. The core idea is to use information from the previous T time-steps to calculate the expected output at time-step T. We also propose a Multi-Threshold Neuron and the corresponding Parallel Parameter normalization to address the challenge of large time steps needed for high accuracy, aiming to reduce network latency and power consumption. Our experimental results demonstrate that our approach achieves state-of-the-art performance. For example, we achieve a top-1 accuracy of 88.60\% with only a 1\% loss in accuracy using 4 time steps while consuming only 35\% of the original power of the Transformer. To our knowledge, this is the first successful Artificial Neural Network (ANN) to SNN conversion for Spiking Transformers that achieves high accuracy, low latency, and low power consumption on complex datasets. The source codes of the proposed method are available at https://github.com/h-z-h-cell/Transformer-to-SNN-ECMT.
\end{abstract}
%%
%% The code below is generated by the tool at http://dl.acm.org/ccs.cfm.
%% Please copy and paste the code instead of the example below.
%%
%%
%% Keywords. The author(s) should pick words that accurately describe
%% the work being presented. Separate the keywords with commas.
%% A "teaser" image appears between the author and affiliation
%% information and the body of the document, and typically spans the
%% page.

% \begin{teaserfigure}
%   \includegraphics[width=\textwidth]{sampleteaser}
%   \caption{Seattle Mariners at Spring Training, 2010.}
%   \Description{Enjoying the baseball game from the third-base
%   seats. Ichiro Suzuki preparing to bat.}
%   \label{fig:teaser}
% \end{teaserfigure}

% \received{20 February 2007}
% \received[revised]{12 March 2009}
% \received[accepted]{5 June 2009}

%%
%% This command processes the author and affiliation and title
%% information and builds the first part of the formatted document.
\maketitle

\section{Introduction}
Spiking neural networks(SNNs) are a type of neural network model that imitates the mechanisms of biological neurons\cite{bohte2000spikeprop,gerstner2014neuronal}. They are called the third generation of neural networks~\cite{maass1997networks} due to their biological plausibility and computational efficiency\cite{tavanaei2019deep,zenke2021visualizing}.
% Unlike traditional neural networks, SNNs concentrate on the generation and reception of spikes. 
Neurons in SNNs do not produce output at each time step. Instead, they become active and emit spikes only when their membrane potential reaches a specific threshold. The sparse activity of spikes leads to significantly higher computational efficiency than traditional neural networks \cite{roy2019towards}, especially on neuromorphic chips \cite{merolla2014million, davies2018loihi, debole2019truenorth, pei2019towards}. However, training large-scale, high-precision, and low-latency SNNs remains challenging due to the non-differentiable nature of spikes.

Currently, there are two main approaches to train SNNs. 
The first approach is direct training using backpropagation or local learning\cite{lee2016training,wu2018spatio,neftci2019surrogate,zhang2020temporal,fang2021deep,wu2022brain,zhu2022efficient,yin2024understanding,xu2024adaptive,duan2022temporal}. These methods utilize differentiable continuous functions or spike-time-dependent plasticity strategies to replace the non-differentiable spike emission rules. However, this training process still relies on standard GPUs that are not well-suited for the unique characteristics of SNNs, leading to significant resource consumption and limited performance.
The second approach is ANN to SNN conversion \cite{cao2015spiking,rueckauer2017conversion,li2021free,deng2021optimal,bu2022optimal}. This conversion method does not require any additional training. Instead, it uses pre-trained ANNs and replaces activation functions with spiking neurons, leveraging the similarity between ReLU activation functions and spike emission rates of integrate-and-fire models. The result SNN model preserves the original ANN's performance but often leads to longer inference times, and the modules that can be successfully converted are limited.

As is well known, Transformers have demonstrated exceptional performance in various vision tasks \cite{dosovitskiy2021an, radford2021learning, jain2023oneformer, liu2023efficientvit, chen2024multiscale}. Despite numerous efforts to convert CNNs to SNNs, converting Transformer models remains a challenge. This is due to unique nonlinear modules such as layernorm and GELU in Transformers that differ from the ReLU function in CNNs. These modules require interaction between neurons within the same layer and exhibit non-linear characteristics, making it challenging to achieve accurate conversion through the linear piecewise quantization of individual neurons.

This paper proposes a new method to convert Transformers to SNNs. The main challenge lies in handling non-linear modules. To address this, we propose an Expectation Compensation Module (ECM) that calculates expectations and replaces each non-linear module. Specifically, a customized ECM is employed to substitute the matrix product, performing most operations through accumulations. This reduces power consumption and ensures the total output matches the expected result at each time step. To improve the efficiency of minimal spikes, we introduce Multi-Threshold Neurons and the corresponding Parallel Parameter normalization, significantly reducing the required latency and power consumption for inference with comparable accuracy.

Our main contributions are summarized as follows:
\begin{itemize}
    \item We analyze the challenges of non-linear module conversion in Transformer and introduce a novel solution called the Expectation Compensation Module, which uses the information from the previous time steps to calculate the expected output at the current time step. This module overcomes the limitations of traditional methods with minimal power consumption increase.
    \item To overcome the issue of slow accuracy improvement over time during Transformer conversion, we propose a Multi-Threshold Neuron and the corresponding Parallel Parameter normalization, substantially reducing power consumption requirements and significantly decreasing latency. 
    \item The proposed method is effective on the ImageNet1k dataset, outperforming existing SNN models in accuracy and significantly reducing power consumption compared to other Transformer models. It achieves a top-1 accuracy of 88.60\%, with only a 1\% accuracy loss compared to ANN, while reducing energy consumption by 65\%.
\end{itemize}

\section{Related Works}
    \subsection*{ANN-SNN Conversion}
    The ANN-SNN conversion methods aim to replicate the performance of ANNs by converting pre-trained ANN weights into synaptic weights of SNNs. Cao et al.~\cite{cao2015spiking} initially proposed training an ANN with ReLU activation function and then replacing the activation layer with IF neurons. 
    Diehl et al.~\cite{diehl2015fast} further narrowed the gap by scaling and normalizing the weights. To address the spike count errors resulting from the hard reset mechanism, soft reset neurons were proposed in Rueckauer et al.~\cite{rueckauer2017conversion} and Han et al.~\cite{han2020rmp}.

    Further research has aimed to minimize conversion errors through various optimization:
    (1) Optimizing thresholds: Sengupta et al.~\cite{sengupta2019going} and Zhang et al.~\cite{zhang2023low} proposed dynamic threshold adjustment strategies during the conversion process. 
    (2) Optimizing membrane potential: Bu et al.~\cite{bu2022optimized} demonstrated that setting the initial membrane potential at half the threshold can reduce errors. Hao et al.~\cite{hao2023reducing} further suggested analyzing residual membrane potential to eliminate conversion errors.
    (3) Optimizing the pre-conversion ANN structure: Esser et al.~\cite{esser2019learned} suggested training ANNs with quantized activation values. Ho and Chang~\cite{ho2021tcl} introduced a trainable clipping layer (TCL) for threshold determination. Ding et al.~\cite{ding2021optimal} proposed a rate norm layer, while others ~\cite{bu2022optimal,han2023symmetric,jiang2023unified,wang2023new} suggested various activation functions to replace ReLU.
    (4) Optimizing spiking neuronal models. Li et al.~\cite{li2022efficient} introduced a neuron model for burst spikes. Wang et al.~\cite{wang2022signed} proposed a memory-enhanced signed neuron. Li et al.~\cite{li2022quantization} suggested incorporating negative spikes and extending simulation time to improve accuracy with minimal cost.
    
    Previous methods for converting CNNs to SNNs were limited by CNN performance.
    Jiang et al.\cite{jiang2024spatiotemporal} introduced Universal Group Operators and a Temporal-Corrective Self-Attention Layer to approximate original Transformers but faced long inference latency and accuracy gaps with the ANN. 
    
    In contrast, this paper presents a new method for converting Transformers to SNNs, achieving high accuracy and low latency while reducing network energy consumption.

    \subsection*{Directly Trained Transformer in ANNs and SNNs}
    The Transformer architecture has performed well in the ANN and SNN domains. Initially, Transformers gained prominence in the ANNs with their self-attention mechanisms as proposed by Vaswani et al.\cite{vaswani2017attention}. Dosovitskiy et al.~\cite{dosovitskiy2021an} then introduced the Vision Transformer (ViT), which divided images into fixed-size patches as token inputs, achieving significant success in computer vision. Fang et al.~\cite{fang2023eva,fang2023eva2} and Sun et al.~\cite{sun2023eva} further expanded ViT models to one billion parameters, pushing the limits of large-scale visual models.
    
    In the SNN domain, spike-based Transformers quickly emerged, incorporating spike self-attention mechanisms with some floating-point calculations~\cite{zhou2023spikformer,li2022spikeformer}. Subsequently, Zhou et al.~\cite{zhou2023spikingformer}, Yao et al.~\cite{yao2024spike} and Shi et al.\cite{shi2024spikingresformer} introduced fully event-based Transformers. Wang et al.\cite{wang2023masked} first trained a modified Transformer and then converted it into a Spiking Transformer.
    Spike-based Transformers have successfully applied to applications, such as monocular depth estimation~\cite{zhang2022spike}, single-object tracking with event cameras~\cite{zhang2022spiking}, and automatic speech recognition~\cite{wang2023complex}. 

    In contrast to previous methods that training Transformers from scratch, this paper focuses on converting pre-trained Transformers into SNNs to reduce energy while maintaining performance.

\section{Preliminaries}
    In this section, we first detail the theoretical basis of the conversion process from ANNs to SNNs. Then, we introduce the Vision Transformer (ViT), the ANN architecture we selected for conversion.
    \subsection{ANN-SNN conversion theory}\label{conversion_theory}
    \subsubsection{Neurons in ANNs.}
    In ANNs, for linear or convolution layers in CNNs using the ReLU activation, the output $\bm{a}^l$ of neurons in layer $l$ can be formulated as:
    \begin{align}\label{ann_relu}
        \bm{a}^l={\rm ReLU}(\bm{W}^l \bm{a}^{l-1})=\max(\bm{W}^l \bm{a}^{l-1},0),
    \end{align}
    where $\bm{W}^l$ represents the weights of the linear transformation or convolution in this layer.
    \subsubsection{Integrate-and-Fire Neurons in SNNs.}
    For Integrate-and-Fire (IF) neurons in SNNs, let $\bm{m}^l(t)$ and $\bm{v}^l(t)$ denote the membrane potential of neurons in the $l$-th layer before and after firing spikes at time-step $t$,  the neural dynamic can be formulated as follows:
    \begin{align}\label{snn1}
        \bm{m}^l(t)&=\bm{v}^l(t-1)+\bm{W}^l \bm{x}^{l-1}(t),\\
        \label{snn2}
        \bm{s}^l(t)&=H(\bm{m}^l(t)-\theta^l),\\
        \bm{x}^l(t)&=\theta^l\bm{s}^l(t),\\
        \label{snn3}
        \bm{v}^l(t)&=\bm{m}^l(t)-\bm{x}^l(t).
    \end{align}
    where $H$ is the Heaviside step function and $\theta^l$ is the neuron threshold in layer $l$. $\bm{s}^l(t)$ is the output spike of layer $l$. $\bm{x}^l(t)$ is the postsynaptic potential and theoretical output of layer $l$, which equals $\theta^l$ if the neuron fires and $0$ otherwise. Following \cite{rueckauer2017conversion} and \cite{han2020rmp}, we use the "reset-by-subtraction" mechanism, where $\bm{v}^l(t)$ decreases by a value of $\theta^l$ if the neuron fires.
    \subsubsection{ANN-SNN Conversion.}
    Combining Equations \eqref{snn1}-\eqref{snn3}, we get
    \begin{equation}\label{snn4}
    \bm{v}^l(t)-\bm{v}^l(t-1)=\bm{W}^l \bm{x}^{l-1}(t)- \bm{x}^l(t).
    \end{equation}
    Summing from time-step 1 to time-step $T$, we have
    \begin{equation}\label{snn5}
        \frac{\bm{v}^l(T)-\bm{v}^l(0)}{T}=\frac{\bm{W}^l\sum_{i=1}^T \bm{x}^{l-1}(i)}{T}-\frac{\sum_{i=1}^T \bm{x}^l(i)}{T}.
    \end{equation}
    Letting $\Phi^l(T)=\frac{\sum_{i=1}^T \bm{x}^{l}(i)}{T}$, we have
    \begin{equation}\label{snn6}
        \Phi^l(T)=\bm{W}^l\Phi^{l-1}(T)-\frac{\bm{v}^l(T)-\bm{v}^l(0)}{T}.
    \end{equation}
    Comparing Equations \eqref{ann_relu} and \eqref{snn6}, $\frac{\bm{v}^l(T)-\bm{v}^l(0)}{T}$ tends to 0 as $T$ becomes large. This allows $\Phi^l(T)$ in SNNs to approximate $\bm{a}^l$ in ANNs.
    \subsubsection{Parameter normalization.}
    Due to the spike-based communication in SNNs, approximation errors arise since SNN neurons can emit only one spike per time step, limited to a firing rate in the range of $[0,r_{\text{max}}]$, where ANNs do not have such constraints. 
    To prevent approximation errors from excessively low or high firing rates, Diehl et al.\cite{diehl2015fast} and Rueckauer et al.\cite{rueckauer2017conversion} introduced weight normalization to rescale parameters using the following equations:
    \begin{align}\label{normalize}
        W^l_\text{SNN}=W^l_\text{ANN}\frac{\lambda^{l-1}}{\lambda^l}.
    \end{align}
    where $\lambda^l$ is determined by the $p$-th percentile of the total activity distribution of layer $l$. 
    Modifying Equation \eqref{normalize} and setting $\theta^l_j$ to 1 is equivalent to adjusting the firing threshold on the soft-reset neuron to $\lambda^l$\cite{bu2022optimized}. This adjustment ensures that the output $\bm{x}^l(t)$ is a spike matrix equal to $\bm{s}^l(t)$ and suits the operational dynamics of SNNs.
    \subsection{Vision Transformer}
    Vision Transformer (ViT) architecture consists of three core components: Embeddings, Transformer Encoder, and Classification Head.
    \subsubsection{Embeddings.}
    The process starts by segmenting an image into patches of specific dimensions, viewing them as a sequence of tokens. Each patch undergoes linear embedding with added positional embeddings, enriching the output token vectors with the patch's content and location within the image.
    \subsubsection{Transformer Encoder.}
    Central to feature extraction, the Transformer Encoder plays a crucial role in various visual tasks. It is divided into two primary segments:
    
    (1) Self-Attention Mechanism. 
    This mechanism calculates a weighted sum of all the values $V$ in a given sequence. The attention weights are determined based on the similarity between a query $Q$ and a key $K$. 
    The values $Q$, $K$, and $V$ are obtained through the input $X$ using weight matrices $W^Q$, $W^K$, and $W^V$ respectively. The following equation describes the matrix form of the output calculation for the self-attention mechanism:
    \begin{align}
        O=\text{Softmax}\left(\frac{Q^T K}{\sqrt{d}}V\right)=\text{Softmax}\left(\frac{{(W^QX)}^T {W^KX}}{\sqrt{d}}{W^VX}\right).
    \end{align} 
    where $d$ is the dimension of the key and query vectors.
    
    (2) Feed-Forward Network. Here, the input vector passes through two linear layers and is activated by the GELU function between them.
    \subsubsection{Classification Head.}
    Features related to the CLS token are directed toward the classification head, which then computes the probabilities for the various classes.

\section{Method}
    In this section, we first analyze the main errors encountered in ANN-SNN conversion. Following this, we propose the Expectation Compensation Module (EC) to preserve the accuracy of non-linear modules. In particular, we detailed a lossless conversion method for the matrix product layer, mainly using additional operations. 
    Additionally, a Multi-Threshold Neuron (MT) is designed to improve the efficiency of minimal spikes, which significantly reduces network latency and energy consumption.
    The diagram shown in Figure~\ref{arch} provides an overview of the architecture we utilized.
    
    \begin{figure}[htb]
        \centering
        \includegraphics[width=\linewidth]{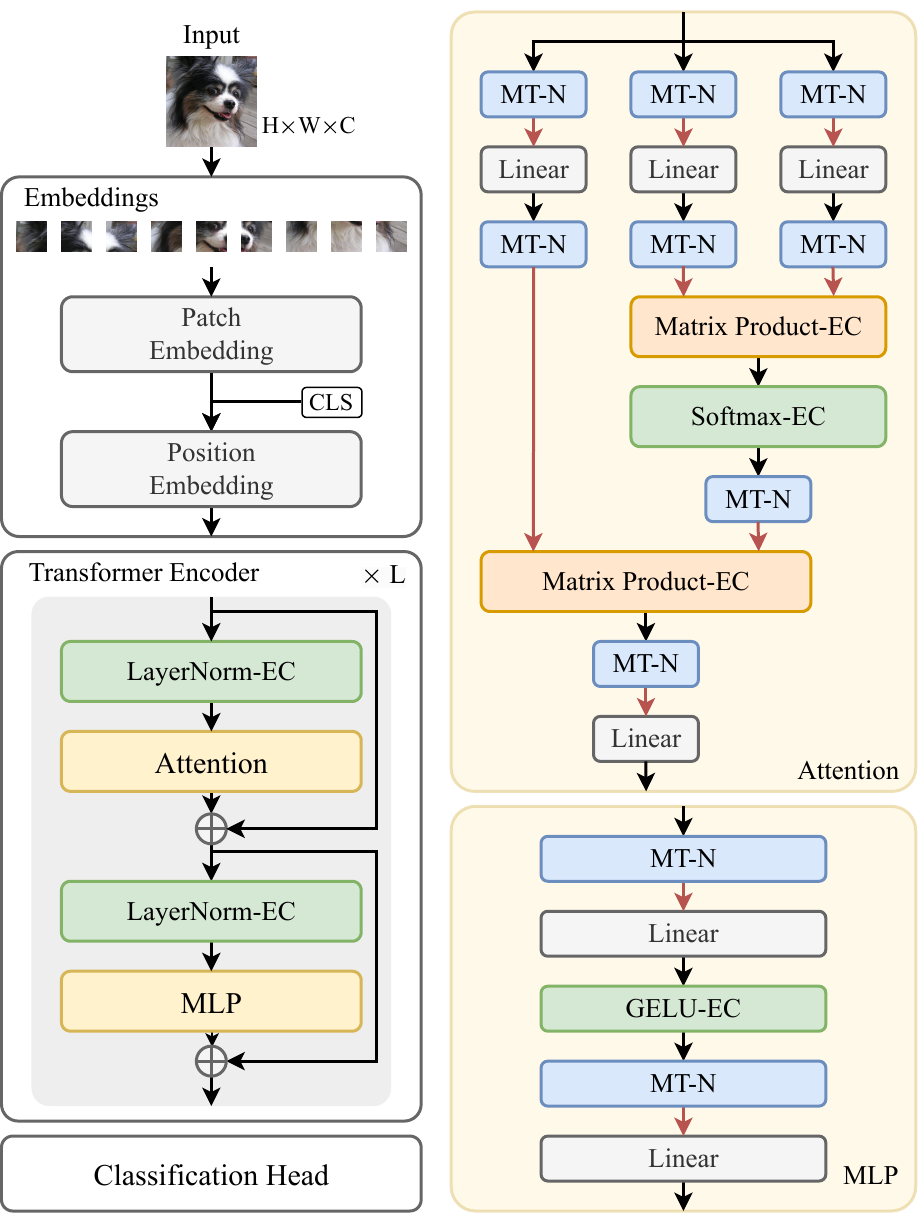}
        \caption{An overview of the proposed architecture, including the whole architecture, Attention, and MLP module.}
        \label{arch}
    \end{figure}
    
    \subsection{Error Analysis of Nonlinear Module in ANN-SNN Conversion}
    
    Existing ANN-SNN conversion methods mainly focus on CNNs, which typically employ linear operations, such as linear transformations and convolutions, combined with ReLU activation, as formulated in Equation \eqref{ann_relu}. However, Transformer architecture uses many non-linear operations, such as GELU, softmax, layernorm, and matrix product, which cannot be directly formulated using Equation \eqref{ann_relu}. Consequently, the current conversion theory discussed in Section~\ref{conversion_theory} does not apply to Transformers, which can lead to conversion errors.

    To be specific, we assume that the outputs of layer $l-1$ in both ANNs and SNNs are identical, denoted as $\bm{a}^{l-1}=\Phi^{l-1}(T)=\frac{\sum_{t=1}^T\bm{x}^{l-1}(t)}{T}$, and we will compare the outputs $\bm{a}^{l}$ and $\Phi^{l}$ in layer $l$. 
    
    Considering an arbitrary non-linear module in layer $l$ of an ANN, its function can be formulated as:
    \begin{align}\label{general_function}
        \bm{a}^l=F(\bm{a}^{l-1}),
    \end{align}
    where $F$ is the function of this layer. 
    Obviously, it cannot be expressed equivalently using Equation \eqref{ann_relu}.
    In this case, if we do not introduce a further conversion method for this non-linear module, the actual output of the SNN counterpart at time $t$ will be $\bm{x}^{l}(t) = F(\bm{x}^{l-1}(t))$. The average output can be formulated as follows:
    \begin{align}\label{actout}
        \Phi^{l}(T)=\frac{\sum_{t=1}^T \bm{x}^l(t)}{T} = \frac{\sum_{t=1}^T F(\bm{x}^{l-1}(t))}{T}.
    \end{align}

    However, in the case of ANNs, the expected average output can be formulated as:
    \begin{align}\label{expout}
        \bm{a}^{l}= F(\bm{a}^{l-1})=F\left(\frac{\sum_{t=1}^T \bm{x}^{l-1}(t)}T\right).
    \end{align}
    Due to the non-linear nature of the module, we have:
    \begin{align}\label{neq}
        \frac{\sum_{t=1}^TF(\bm{x}^{l-1}(t))}T \neq F\left(\frac{\sum_{t=1}^T\bm{x}^{l-1}(t)}T\right).
    \end{align}
    This implies that the output $\Phi^{l}(T)$ of SNNs in Equation \eqref{actout} is not equivalent to the output $\bm{a}^{l}$ of ANNs in Equation \eqref{expout}, posing challenges for non-linear conversion.
    % \begin{figure}[t]
    %     \centering
    %     \includegraphics[width=0.5\textwidth]{pic/matrix_product.png}
    %     \caption{The Matrix Product-EC Diagram represents a sequence of operations. It begins with input matrices $A$ and $B$ and three accumulators, namely $S_A$, $S_B$ and $S_K$. The black arrows denote the sub-matrix products, while the blue and red arrows update the accumulators, leading to the final output.}
    %     \label{matrixproduct}
    % \end{figure}
    
    \begin{figure}[htb]
        \centering
        \includegraphics[width=\linewidth]{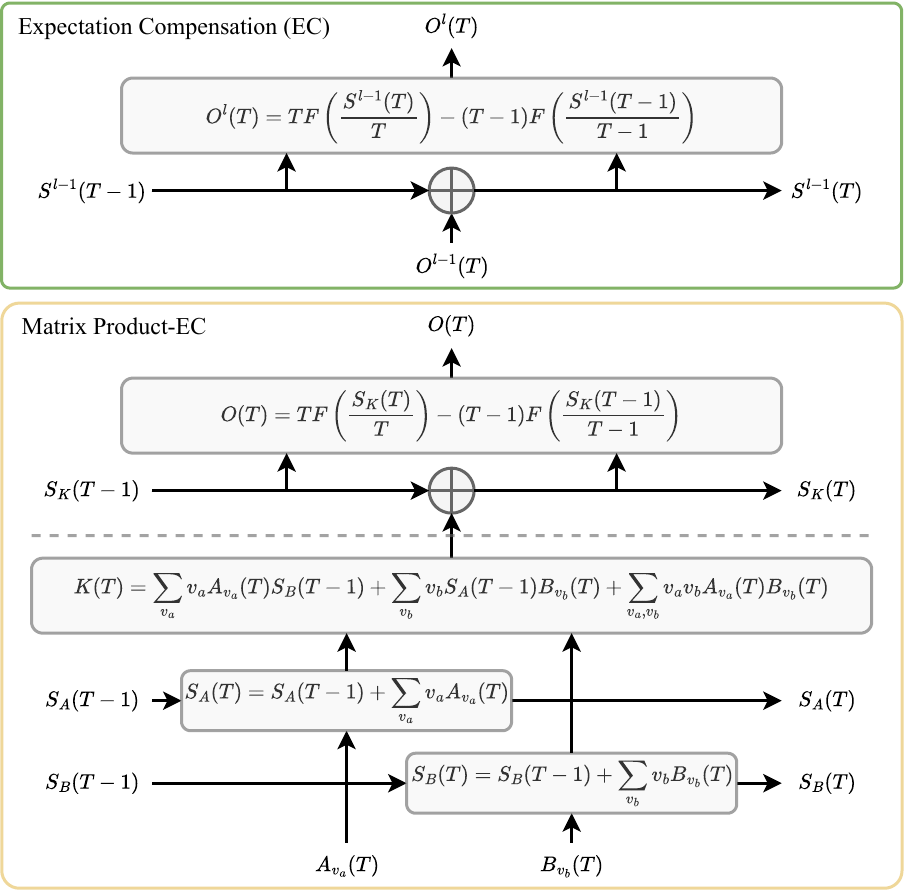}
        \caption{The upper diagram shows the general Expectation Compensation module(EC). The lower diagram shows the Expectation Compensation module for Matrix Product(Matrix Product-EC).}
        
        \label{pic3}
    \end{figure}
   \subsection{Expectation Compensation Module}
    To overcome the challenge of converting non-linear layers, we propose using Expectation Compensation Modules to preserve non-linearity throughout the conversion process by leveraging prior information to compute expectations.
    \subsubsection{General Expectation Compensation Module}\ 
    
    The theorem below calculates the expected output of the arbitrary non-linear layer at each time step in SNNs.
    \begin{theorem}\label{theorem1}
        Consider a non-linear layer $l$ with a function $F$. In SNNs, the output of this layer at time $t$ is denoted as $\bm{O}^{l}(t)$. Let $\bm{S}^l(T)$ be the cumulative sum of layer $l$ outputs up to time $T$, given by $\bm{S}^l(T)=\sum_{t=1}^T\bm{O}^{l}(t)$. The expected output of the SNNs at time $T$ is given by:
        \begin{equation}\label{snn_out}
            \bm{O}^{l}(T)=TF\left(\frac{\bm{S}^{l-1}(T)}{T}\right)-(T-1)F\left(\frac{\bm{S}^{l-1}(T-1)}{T-1}\right).
        \end{equation}
    \end{theorem}
     The detailed proof is provided in the supplementary materials.
     Theorem ~\ref{theorem1} indicates that lossless conversion can be achieved by an accumulator to records $\bm{S}^{l-1}(T)$ and an optional variable to records $TF\left(\bm{S}^{l-1}(T)/T\right)$ as shown in Figure ~\ref{pic3}.

    % \begin{figure}[htb]
    %     \centering
    %     \includegraphics[width=0.5\textwidth]{pic/pic4.drawio.png}
    %     \caption{xxx}
    %     \label{pic4}
    % \end{figure}
    \subsubsection{Expectation Compensation Module for Matrix Product}\ 
    
    For the matrix product layer, we can convert it into a specialized module that primarily uses additional operations to achieve lossless conversion.
    The theorem below outlines how to calculate the expected output of the matrix product layer at each time step in SNNs.
    \begin{theorem}\label{theorem2}
        Consider a module for matrix product that receives two sets of spike inputs, denoted by $\bm{A}_{v_a}(t)$ and $\bm{B}_{v_b}(t)$. These inputs are generated by neurons $A$ and $B$, respectively, and are characterized by multiple thresholds $v_a$ and $v_b$, as described in Section 4.3. 

        We can integrate the input by $\bm{A}(t)=\sum_{v_a}v_a\bm{A}_{v_a}(t)$ and $\bm{B}(t)=\sum_{v_b}v_b\bm{B}_{v_b}(t)$. Here, $\bm{A}(t)$ and $\bm{B}(t)$ are the sum matrices weighted by multiple thresholds $v_a$ and $v_b$, respectively.

        Let $\bm{S}_{A}(T)=\sum_{t=1}^TA(t)$ and $\bm{S}_{B}(T)=\sum_{t=1}^TB(t)$ represent the cumulative sum of inputs up to time $T$. We define $\bm{S}_{K}(T)=\bm{S}_{A}(T)\bm{S}_{B}(T)$. 
        Then, the expected output at time T can be formulated as:
        \begin{equation}\label{matrix_exp_all}
            \begin{aligned}
                \bm{O}(T)=\frac1T\bm{S}_{K}(T)-\frac1{T-1}\bm{S}_{K}(T-1),
            \end{aligned}
        \end{equation}
        where $\bm{S}_{K}(T)$ can be calculated mainly using addition, as described by the following equation:
        \begin{equation}\label{matrix_exp}
            \bm{S}_{K}(T)=\bm{S}_{K}(T-1)+\bm{K}(T)\\
        \end{equation}
        \begin{equation}
            \begin{aligned}
                \bm{K}(T)=&\sum_{v_a,v_b}v_av_b\bm{A}_{v_a}(T)\bm{B}_{v_b}(T)+\sum_{v_a}v_a\bm{A}_{v_a}(T)\bm{S}_B(T-1)
                \\&+\sum_{v_b}v_b\bm{S}_A(T-1)\bm{B}_{v_b}(T).
            \end{aligned}
        \end{equation}
    \end{theorem}
    The detailed proof is provided in the supplementary materials.
    According to Theorem ~\ref{theorem2}, the output $\bm{O}(T)$ can be obtained through the process illustrated in Figure~\ref*{pic3}. The main power consumption in this process occurs during the matrix product calculation of $\bm{K}(T)$ using spike matrices, which can be implemented through accumulations.
    Since each position of the input matrix has only one effective threshold at each time, it limits the total number of input spikes, thereby restricting the total number of operations. Combined with the sparsity of spikes, this reduces power consumption at each time step while achieving lossless conversion.
    \subsection{Multi-Threshold Neuron}

    \begin{figure}[t]
        \centering
        \includegraphics[width=\linewidth]{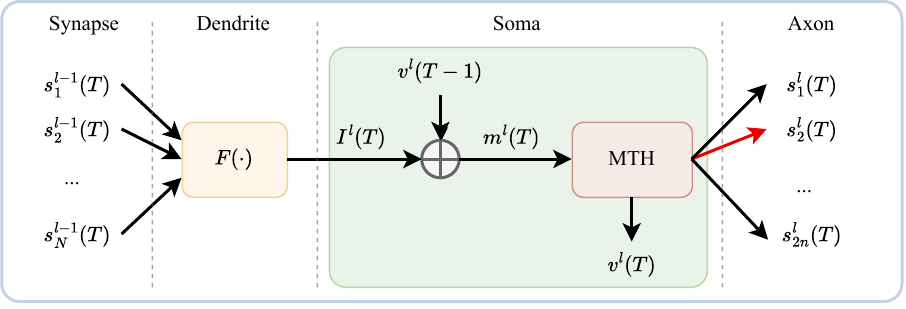}
        \caption{Diagram of MT neuron. MT neuron receives input from nonlinear/linear modules and emits up to one spike.}
        \label{pic5}
    \end{figure}

    \subsubsection{Problem of Consumption and Latency}\ 
    
    If we only use the Expectation Compensation Module, neuron communication will remain in a floating-point format. As discussed in Section \ref{Energy_Estimation}, most of the network's power consumption occurs in the linear and matrix product layers.
    To reduce the network's energy consumption, we introduce spiking neurons before each linear layer and matrix product layer. Thus, we can significantly reduce the network's power consumption by adopting spiking communication.
    
    However, if we only use one threshold, no matter how set, it will result in excessively high firing rates or high inference latency.
    The findings in Section \ref{neurontest} demonstrate the importance of having large and small thresholds in the Transformer. 
    
    \subsubsection{The Proposed Solution: Multi-Threshold Neuron}\label{MTneuron}\ 
    
    To tackle the challenges of high power consumption and latency, we propose a Multi-Threshold Neuron (MT neuron). 
    
    This neuron model has additional thresholds built upon the base threshold, allowing it to process more information in a single time step.
    The MT neuron is characterized by parameters including the positive and negative base thresholds, represented as $\theta_1$ and $-\theta_2$, respectively, and the number of thresholds denoted as $2n$.
    We can refer to $\lambda^l_p$ as the $p$-th threshold value of the MT neuron corresponding to index $p$.
    \begin{align}
        \begin{aligned}
        &\lambda^l_1=\theta^l_1, \lambda^l_2=2\theta^l_1, ..., \lambda^l_n = 2^{n-1}\theta^l_1, \\&\lambda^l_{n+1}=-\theta^l_2, \lambda^l_{n+2}=-2\theta^l_2, ..., \lambda^l_{2n} = -2^{n-1}\theta^l_2, 
        \end{aligned}
    \end{align}
    As shown in Figure ~\ref{pic5}, the dynamic of MT neurons is described by:
    \begin{align}
          &\quad I^l_j(t)= F^l_j(\bm{s}_{,1}^{l-1}(t),...,\bm{s}_{,2n}^{l-1}(t)),
        \\&\quad m^l_j(t)=v^l_j(t-1)+I^l_j(t),
        \\&\quad s^l_{j,p}(t)=MTH_{\theta_1,\theta_2,n}(m^l_j(t))
        \\&\quad x^{l}_j(t)=\sum_{p}s^{l}_{j,p}(t)\lambda^l_p,
        \\&\quad v^l_j(t)=m^l_j(t)-x^l_j(t).
    \end{align}
    % $\bm{x}^l(t)$ is the unweighted postsynaptic potential and output of layer $l$
    The variables $I^l_j(t)$,$s^l_j(t)$,$x^l_j(t)$, $m^l_j(t)$ and $v^l_j(t)$ respectively represent the input, output, postsynaptic potential, and the membrane potential before and after spikes of the $j$-th neuron in the $l$-th layer at time $t$.
    Meanwhile, $F$ is a linear or nonlinear function of this layer. The function $MTH_{\theta_1,\theta_2,n}(x)$ can be described using the following piecewise function:
    \begin{align}
        \begin{aligned}
            &MTH_{\theta_1,\theta_2,n}(x):
            \\&\left\{
            \begin{array}{r@{}l@{}l}
            \lambda^l_n-\frac{\lambda^l_1}{2}\leq&x:&s^l_{j,n}(t)=1,
            \\\lambda^l_{n-1}-\frac{\lambda^l_1}{2}\leq&x<\lambda^l_n-\frac{\lambda^l_1}{2}:&s^l_{j,n-1}(t)=1,
            \\...&&...
            \\\frac{\lambda^l_1}{2}\leq&x<\lambda^l_{2}-\frac{\lambda^l_1}{2}:&s^l_{j,1}(t)=1,
            \\\frac{\lambda^l_{n+1}}{2}\leq&x< \frac{\lambda^l_1}{2}:& all\ the\ s^l_{j,p}(t) = 0,
            \\\lambda^l_{n+2}-\frac{\lambda^l_{n+1}}{2}\leq&x<\frac{\lambda^l_{n+1}}{2}:&s^l_{j,n+1}(t)=1,
            \\...&&...
            \\\lambda^l_{2n}-\frac{\lambda^l_{n+1}}{2}\leq&x<\lambda^l_{2n-1}-\frac{\lambda^l_{n+1}}{2}:&s^l_{j,2n-1}(t)=1,
            \\&x<\lambda^l_{2n}-\frac{\lambda^l_{n+1}}{2}:&s^l_{j,2n}(t)=1.
            \end{array}\right.
        \end{aligned}
    \end{align}
    
    \begin{figure}[t]
        \centering
        \includegraphics[width=0.9\linewidth]{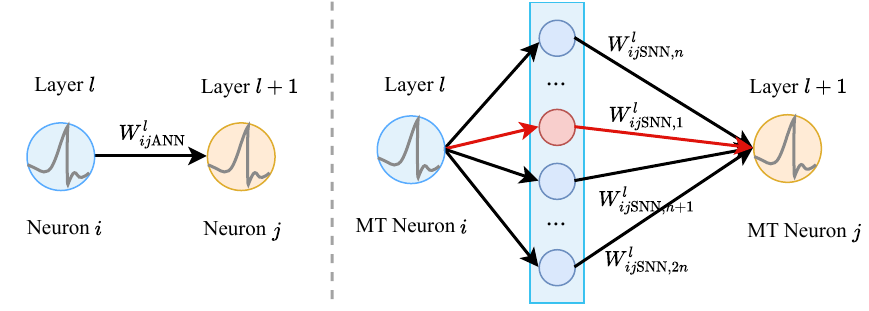}
        \caption{Left: Original connection in ANN. Right: Parallel Parameter normalization of MT neuron in SNN. The MT Neuron extends one connection to $2n$ channels. At each time, only one of the $2n$ channels can emit a spike.}
        \label{pic2}
    \end{figure}
    
    The results of experiments presented in Section \ref*{neurontest} indicate that although this neuron has multiple thresholds, most of the spikes it generated are concentrated in $\theta_1$ and $-\theta_2$. The spikes generated by other thresholds are minimal, which reduces energy consumption and inference latency.

    \subsubsection{Parallel Parameter normalization for MT Neuron}\label{PPNMT}\ 
    
    Spike neurons communicate with each other by producing an output spike of either 0 or 1. As for function $F$ in Figure ~\ref{pic5}.
    
    If $F$ is a Matrix Product-EC function, we only need to send spikes $s^l(t)$ to $F$ as $\bm{A}_{v_a}(t)$ or $\bm{B}_{v_b}(t)$.
    
    If $F$ is a general nonlinear EC function, we will integrate spike output by $I_j^l(t)=F^l_j(\sum_{p}\bm{s}_{,p}^{l-1}(t)\lambda^{l-1}_p)$.
    
    If $F$ is a linear function, $I_j^l(t)$ can be expressed by
    \begin{equation}
        I^l_j(t)=\sum_{i}w^l_{ij\text{ANN}}x^{l-1}_i(t)
        =\sum_{i}w^l_{ij\text{ANN}}\sum_{p}s^{l-1}_{i,p}(t)\lambda^{l-1}_p
    \end{equation}
    A parallel parameter normalization method is proposed to support spike communication between MT neurons in a linear layer.
    This method extends the ANN weight to 2n weights in the SNN corresponding to 2n thresholds of MT neurons, as shown in Figure ~\ref{pic2}. We update these weights using the following formula:
    % \begin{align}
    %     \begin{aligned}
    %     &\lambda^l_1=\theta^l_1, \lambda^l_2=2\theta^l_1, ..., \theta^l_n = 2^{n-1}\theta^l_1, \\&\lambda^l_{n+1}=-\theta^l_2, \lambda^l_{n+2}=-2\theta^l_2, ..., \lambda^l_{2n} = -2^{n-1}\theta^l_2, 
    %     \end{aligned}
    % \end{align}
    % Here, $\lambda^l_p$ refers to the $p$-th threshold value of the MT neuron.
    \begin{align}
        W^l_{\text{SNN},p} = W^l_{\text{ANN}}\frac{\lambda^{l-1}_p}{\lambda^l_1}
    \end{align}
    Here, we divide an extra variable $\lambda^l_1$ to equilibrate parameter size. Let's set $\eta^l = \frac{\theta_2^l}{\theta_1^l}$. This brings the neuron to an equivalent form, which is as follows:
    \begin{align}
        &I^l_j(t)= \sum_{i,p}w^l_{ij\text{SNN,p}}s^{l-1}_{i,p}(t)
        \\&\theta_{1,new}=1,\theta_{2,new}=\eta
    \end{align}
    % \begin{align}
    %     \left\{
    %     \begin{array}{r@{}l@{}l}
    %     (2^{n-1}-\frac12)\leq&m^l_j(t):&s^l_{j,n}(t)=1,
    %     \\(2^{n-2}-\frac12)\leq&m^l_j(t)<(2^{n-1}-\frac12):&s^l_{j,n-1}(t)=1,
    %     \\...&&...
    %     \\\frac12\leq&m^l_j(t)<\frac32:&s^l_{j,1}(t)=1,
    %     \\-\frac12\eta^l\leq&m^l_j(t)< \frac12:& all\ the\ s^l_{j,p}(t)\ are\ 0,
    %     \\-\frac32\eta^l\leq&m^l_j(t)<-\frac12\eta^l:&s^l_{j,n+1}(t)=1,
    %     \\...&&...
    %     \\(\frac12-2^{n-1})\eta^l\leq&m^l_j(t)<(\frac12-2^{n-2})\eta^l:&s^l_{j,2n-1}(t)=1,
    %     \\&m^l_j(t)<(\frac12-2^{n-1})\eta^l:&s^l_{j,2n}(t)=1.
    %     \end{array}\right.
    % \end{align}
    % Here, $s^l_{j,p}(t)$ represents the spike output of the $j$ th neuron in layer $l$ on the channel corresponding to the $p$ th threshold value.
    % There is only up to one active threshold each time. This means that only up to one of the $2n$ $s^l_{j,p}(t)$ can transmit a spike at any moment.
    
    Based on the above discussion, we name this method: Expectation Compensation and Multi-Threshold(ECMT). The overall conversion algorithm can be summarized in Algorithm \ref{alg:algorithm}.
    The conversion is a one-time process, allowing the converted model to be reused without other computations before use.

    \begin{algorithm}[htb]
        \caption{The conversion method using Expectation Compensation Module and Multi-Threshold Neuron(ECMT)}
        \label{alg:algorithm}
        \begin{algorithmic}[1] %[1] enables line numbers
            \item[\textbf{Input}:] Pre-trained Transformer ANN model $f_{\text{ANN}}(\bm{W})$; Dataset $D$; Time-step $T$ to test dataset; Threshold percent $p$.
            \item[\textbf{Output:}] SNN model $f_{\text{SNN}}(\bm{W},\bm{\theta}_1,\bm{\theta}_2,\bm{v})$
            \STATE {\bf step1:} Obtain the base thresholds $\bm{\theta}_1$ and $\bm{\theta}_2$
            \FOR{length of Dataset \bm{$D$}}
            \STATE Sample minibatch data from \bm{$D$}
            \STATE Run the data on $f_{\text{ANN}}$ and static the activation values before linear and matrix product module at $p\%$ and $(1-p\%)$, setting them as $\bm{\theta}_1$ and $-\bm{\theta_2}$ respectively.
            \ENDFOR
            \STATE {\bf step2:} Converted to SNN model
            \FOR{module $m$ in $f_{\text{ANN}}.\text{Module}$}
            \IF {$m$ is Linear Module}
            \STATE Add a Multi-Threshold Neuron before $m$
            \ELSIF {$m$ is Matrix Product}
            \STATE replace $m$ by two Multi-Threshold Neurons followed by a Matrix Product EC Module
            \ELSIF {$m$ is Other Nonlinear Module}
            \STATE replace $m$ by an EC Module
            \ENDIF
            \ENDFOR
            \STATE Set the base thresholds of MT neurons to corresponding $\bm{\theta}_1$,$-\bm{\theta_2}$ and set the initial membrane potential $\bm{v}$ to $0$.
            \STATE $f_{\text{SNN}}=\text{Parallel Parameter normalization(}f_{\text{ANN}}\text{)}$
            \STATE \textbf{return} $f_{\text{SNN}}$   
            % \textbackslash{}\textbackslash{} Now it is $f_{SNN}
        \end{algorithmic}
    \end{algorithm}

    \begin{table*}[htb]
        \centering
        \caption{Accuracy and energy consumption ratio of ECMT(Ours) on ImageNet1k dataset}
        \label{accandenergy}
        \begin{tabular}{cccccccccc}
        \toprule
        \multirow{2}{*}{Arch.}&\multirow{2}{*}{Accuracy/Energy}&\multirow{2}{*}{Original (ANN)}&\multicolumn{7}{c}{Ours (SNN)}\\
        \cmidrule(lr){4-10}
        &&&T=1&T=2&T=4&T=6&T=8&T=10&T=12\\
        \midrule
        \multirow{2}{*}{ViT-S/16}&Acc.~(\%)&78.04&0.17&10.66&62.85&73.22&76.03&77.07&77.41\\
        &Energy ratio&1&0.06&0.15&0.37&0.59&0.82&1.03&1.25\\
        \midrule
        \multirow{2}{*}{ViT-B/16}&Acc.~(\%)&80.77&0.24&20.89&69.98&77.81&79.40&80.12&80.38\\
        &Energy ratio&1&0.04&0.12&0.30&0.48&0.66&0.84&1.01\\
        \midrule
        \multirow{2}{*}{ViT-L/16}&Acc.~(\%)&84.88&3.62&75.38&83.20&84.32&84.60&84.68&84.71\\
        &Energy ratio&1&0.04&0.12&0.27&0.43&0.58&0.74&0.89\\
        \midrule
        \multirow{2}{*}{EVA}&Acc.~(\%)&89.62&2.49&84.08&88.60&89.23&89.40&89.45&89.51\\
        &Energy ratio&1&0.06&0.15&0.35&0.55&0.74&0.93&1.13\\
        \bottomrule
        \end{tabular}
    \end{table*}
    \begin{table*}[htb]
        \centering
        \caption{Comparison between the proposed method and previous works on ImageNet1k dataset}
        \label{cmp}
        \begin{tabular}{cccccc}
        \toprule
        Method & Type & Arch. & Param.~(M) & T & Accuracy (\%) \\
        \midrule
        Spikingformer\cite{zhou2023spikingformer}&Direct Training&Spikingformer-4-384-400E&66.34&4&75.85\\
        Spike-driven Transformer\cite{yao2024spike}&Direct Training&Spiking Transformer-8-768*&66.34&4&77.07\\
        Spikeformer\cite{li2022spikeformer}&Direct Training&Spikeformer-7L/3$\times$2$\times$4&38.75&4&78.31\\
        RMP\cite{han2020rmp}&CNN-to-SNN&VGG-16&138&4096&73.09\\
        SNM\cite{wang2022signed}&CNN-to-SNN&VGG-16&138&64&71.50\\
        TS\cite{deng2021optimal}&CNN-to-SNN&VGG-16&138&64&70.97\\
        QFFS\cite{li2022quantization}&CNN-to-SNN&VGG-16&138&4(8)&72.10(74.36)\\
        \cmidrule(l){3-6}
        \multirow{2}{*}{QCFS\cite{bu2022optimal}}&\multirow{2}{*}{CNN-to-SNN}&ResNet-34&21.8&64&72.35\\
        &&VGG-16&138&64&72.85\\
        \cmidrule(l){3-6}
        \multirow{2}{*}{SRP\cite{hao2023reducing}}&\multirow{2}{*}{CNN-to-SNN}&ResNet-34&21.8&4(64)&66.71(68.61)\\
        &&VGG-16&138&4(64)&66.46(69.43)\\
        \cmidrule(l){3-6}
        MST\cite{wang2023masked}&Transformer-to-SNN&Swin-T(BN)&28.5&128(512)&77.88(78.51)\\
        \cmidrule(l){3-6}
        STA\cite{jiang2024spatiotemporal}&Transformer-to-SNN&ViT-B/32&86&32(256)&78.72(82.79)\\
        \cmidrule(l){3-6}
        \multirow{4}{*}{{\bf ECMT(Ours)}}&\multirow{4}{*}{Transformer-to-SNN}&ViT-S/16&22&8(10)&76.03(77.07)\\
        &&ViT-B/16&86&8(10)&79.40(80.12)\\
        &&ViT-L/16&307&4(8)&83.20(84.60)\\
        &&EVA&1074&4(8)&88.60(89.40)\\
        \bottomrule
        \end{tabular}
    \end{table*}
    % \begin{table*}[htb]
    %     \centering
    %     \caption{Comparison between the proposed method and previous works on ImageNet1k dataset}
    %     \label{cmp}
    %     \begin{tabular}{cccccc}
    %     \toprule
    %     Method & Type & Arch. & Param.~(M) & T & Accuracy (\%) \\
    %     \midrule
    %     \multirow{4}{*}{{\bf ECMT(Ours)}}&\multirow{4}{*}{Transformer-to-SNN}&ViT-S/16&22&8(10)&76.03(77.07)\\
    %     &&ViT-B/16&86&8(10)&79.40(80.12)\\
    %     &&ViT-L/16&307&4(8)&83.20(84.60)\\
    %     &&EVA&1074&4(8)&88.60(89.40)\\
    %     \bottomrule
    %     \end{tabular}
    % \end{table*}
    
\section{Experimental results}
    In this section, we first evaluate the proposed method's performance on the ImageNet dataset. Then, we compare our method with state-of-the-art SNN training and ANN-SNN conversion methods. Additionally, we perform ablation experiments on Multi-Threshold Neurons. Finally, we analyze the power consumption of the SNNs converted by our method.
    
    \subsection{Experimental Setup}
    We convert pre-trained Vision Transformer including the $\text{ViT-S/16}$, $\text{ViT-B/16}$, $\text{ViT-L/16}$ with 224 resolution \cite{vaswani2017attention}, and the $\text{EVA}$ model \cite{fang2023eva} on Imagenet1k dataset \cite{deng2009imagenet}. %eva\_g\_patch14
    For all Multi-Threshold Neurons, we set $n$ to 8 for $\text{ViT-S/16}$, $\text{ViT-B/16}$, $\text{ViT-L/16}$ and 6 for $\text{EVA}$. And we set threshold percent $p$ to 99. 
    A more detailed setup can be found in supplementary materials.
    
    \subsection{Experimental results on different model}
    Based on the provided data, Table ~\ref{accandenergy} compares performance metrics for various architectures. The analysis shows that our SNN approach can achieve comparable accuracies to traditional ANNs with few time steps. Notably, there is only a 1\% drop in accuracy observed relative to their ANN counterparts at T=10 for ViT-S/16, T=8 for ViT-B/16, T=6 for ViT-L/16, and as early as T=4 for EVA. This trend highlights the efficiency of our conversion strategy, especially within the larger models.

    Taking a closer look at the EVA model, our method achieves an impressive 88.60\% accuracy at just T=4, with a negligible 1\% accuracy degradation while using only 35\% of the energy required by the equivalent ANN model. These results demonstrate our approach's effectiveness and suggest its potential for significant energy savings without substantially compromising accuracy, particularly in complex and larger-scale model architectures.

    % Additional experiments with more datasets can be found in the supplementary material.
    
    \subsection{Comparison with the State-of-the-art}
    % To further demonstrate the effectiveness of our approach, 
    Our experiments on the ImageNet1k dataset have pushed the frontiers of neural network efficiency and accuracy. Table~\ref{cmp} provides a compelling narrative of our progress. Our method is unique in that it facilitates the conversion of Transformer models into SNNs, and it stands out for its computational frugality and high accuracy yield. This marks a significant stride over previous state-of-the-art methodologies.
    
    Firstly, our method is designed to be more efficient than direct training approaches. Instead of starting from scratch, we leverage large pre-trained models to economize on computational efforts and achieve higher accuracy levels than traditional methods. 
    This approach demonstrates our ability to capitalize on the intrinsic efficiencies of pre-trained networks and apply them successfully to SNNs.

    Secondly, our technique surpasses the CNN-to-SNN conversion methods in every aspect. Remarkably, even with the ViT-S/16 model at just 8 time steps, we have achieved an accuracy of 76.0\%, which outperforms the highest accuracy metrics achieved in previously published CNN-to-SNN works. 
    This highlights the effectiveness of our conversion protocol and confirms its superiority in translating CNN architectures into their spiking counterparts.

    Finally, compared to the Swin-T(BN) transformer-to-SNN conversion method mentioned in \cite{wang2023masked}, our approach does not require specific transformer structures for SNN training. Instead, it enables the direct conversion of mainstream ViT models. When compared to the transformer-to-SNN conversion method in \cite{jiang2024spatiotemporal}, our method can decrease overall energy consumption while requiring extremely lower latency. 
    Based on the above discussion, our process ensures quick turnaround and achieves accuracy within 10 temporal steps.

    We conducted experiments using four different models, ViT-S/16, ViT-B/16, ViT-L/16, and EVA, and found that the accuracies achieved at time steps 8, 8, 4, and 4, respectively, were as follows: 76.03\%, 79.4\%, 83.2\%, and 88.6\%. The EVA model, in particular, performed exceptionally well at reduced time steps, indicating the robustness of our method and its potential to set new benchmarks in SNN performance.

    \subsection{The Effect of Multi-Threshold Neuron\label{neurontest}}
    To verify the effectiveness of the Multi-Threshold Neuron, we conducted an experiment to explore the model by varying the number of thresholds in the neurons. We denoted the number of thresholds as $2n$ and experimented with $n=4$, $n=6$, and $n=8$. Our results, depicted in Figure \ref{Acc}, illustrate that as the value of $n$ increases, more large thresholds are included. This suggests that having large thresholds is crucial for enhancing performance.
    
    We also increased the base threshold to investigate further while keeping $n=8$. This allowed us to study the effect of smaller thresholds by their omission. The results were precise: models without small thresholds performed worse than those with both large and small thresholds. Our results showed that both large and small thresholds are crucial for the model. This emphasizes the need for a larger $n$ to achieve low-latency and high-accuracy conversion.

    \begin{figure}[t]
        \centering
        \includegraphics[width=0.9\linewidth]{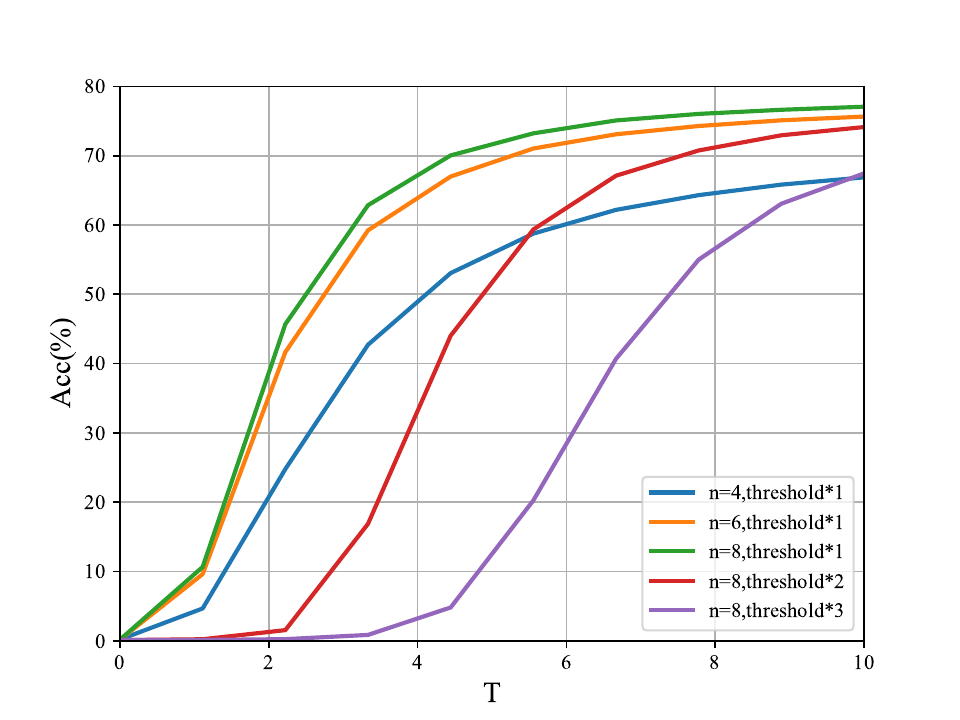}
        \caption{Accuracy under different number and size of thresholds on ViT-S/16, $2n$ denotes the number of thresholds.}
        \label{Acc}
    \end{figure}
    \begin{figure}[t]
        \centering
        \includegraphics[width=\linewidth]{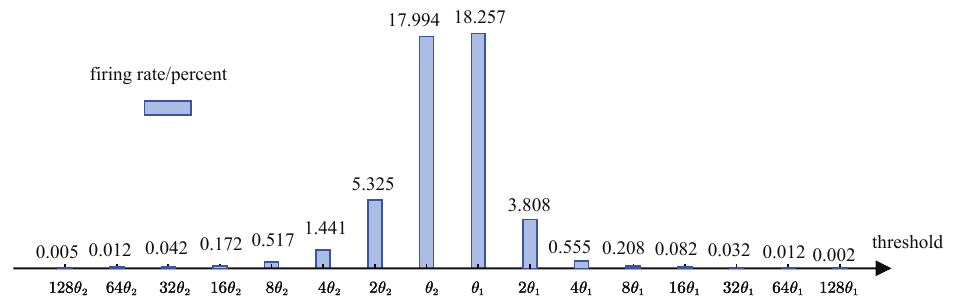}
        \caption{Firing rate at different thresholds}
        \label{firerate}
    \end{figure}
    
    Additionally, we measured the firing rates of spikes associated with each threshold when $n$ was set to 8. The outcomes are presented in Figure \ref{firerate}, which shows that the majority of spikes cluster around the base thresholds, while the spikes generated by other thresholds are minimal. This indicates that adding thresholds consumes less energy but significantly reduces the inference latency.
    
    \subsection{Energy Estimation}\label{Energy_Estimation}
    \begin{table}[t]
        \centering
        \caption{Theoretical calculation dimensions and actual numerical results of different modules, with image patches $N=577$, channels $C=1408$, self-attention heads $N_h=16$, and MLP hidden layer channels $C_h=6144$.}
        \label{energy}
        \resizebox{\linewidth}{!}{
        \begin{tabular}{ccc}
        \toprule
        \multirow{2}{*}{Module}&\multicolumn{2}{c}{Computation}\\
        \cmidrule(l){2-3}
        &Complexity&Results (M)\\
        \midrule
        LayerNorm 1&$N*C$&0.81\\
        Linear $qkv$&$N*C*3C$&3431.65\\
        Matrix Product $q,k$&$Nh*N*(C/Nh)^2$&71.49\\
        Softmax&$Nh*N*N$&5.33\\
        Matrix Product $s,v$&$Nh*N*N*(C/Nh)$&468.76\\
        Linear out&$N*C*C$&1143.88\\
        LayerNorm 2&$N*C$&0.81\\
        MLP Linear 1&$N*C*Ch$&4991.48\\
        GELU&$N*Ch$&3.54\\
        MLP Linear 2 &$N*Ch*C$&4991.48\\
        \bottomrule
        \end{tabular}
        }
    \end{table}
    
    In order to determine the energy consumption of the SNNs, we begin by calculating the theoretical computational complexity for each module presented in the EVA model, as detailed in Table~\ref{energy}. 

    We then employ the formula presented in \cite{rathi2020diet} to estimate the energy consumption of SNNs, as detailed in Equation \eqref{energy1}:

    \begin{equation}\label{energy1}
        \frac{E_{\text{SNN}}}{E_{\text{ANN}}}=\frac{MACs_{\text{SNN}}*E_{\text{MAC}}+ACs_{\text{SNN}}*E_{\text{AC}}}{MACs_{\text{ANN}}*E_{\text{MAC}}}.
    \end{equation}
    Here we set $E_{\text{MAC}}=4.6pJ$ and $E_{\text{AC}}=0.9pJ$ according to \cite{horowitz20141}.
    
    The original network performs most of its computation in linear and matrix product layers. Our method enables us to implement linear transformations of spikes entirely using accumulations and matrix products primarily using accumulations. As a result, we can estimate the number of multiply operations ($MACs_{\text{SNN}}$) to be zero. 
    We evaluated the total energy consumption ratio of our method compared to the original ANNs, and the results are summarized in Table~\ref{accandenergy}. Our method reaches a high accuracy of 88.60\% using only 4 time steps, with a marginal loss of 1\% compared to the original ANNs, while consuming only 35\% of the energy.
    
\section{Conclusion and Discussion}
    In this paper, we propose a novel method for converting pretrained Vision Transformers to SNNs with reduced latency. This approach diverges from previous approaches focusing on converting CNNs to SNNs or directly training SNNs, our method converts pre-trained ViTs to SNNs in a low latency. It replaces various modules with a combination of Expectation Compensation Modules and Multi-Threshold Neurons, achieving significantly higher accuracy on the ImageNet dataset with very low latency compared to previous conversion methods. Moreover, the converted models exhibit substantially less energy consumption than the original ANN ViTs. Our method bridges the performance gap between SNNs and ANNs, paving the way for ultra-high-performance SNNs.
    
    Our research has made significant progress in converting Transformers into SNNs with better performance. 
    However, our current method still requires a small amount of multiplication and cannot use accumulations for implementation alone. 
    Although, this issue can be addressed by utilizing hybrid neural networks such as Zhao et al~\cite{zhao2022framework}, which is based on neuromorphic Tianjic chips~\cite{pei2019towards}.
    Future work may focus on finding alternative solutions for non-linear modules to eliminate the remaining multiplications. 
    This will make them more suitable for conversion and pave the way for further exploration of the conversion from Transformers to SNNs.

\clearpage
%%
%% The acknowledgments section is defined using the "acks" environment
%% (and NOT an unnumbered section). This ensures the proper
%% identification of the section in the article metadata, and the
%% consistent spelling of the heading.
\begin{acks}
This work was supported by the National Natural Science Foundation of China (62176003, 62088102) and the Beijing Nova Program (20230484362).
\end{acks}

%%
%% The next two lines define the bibliography style to be used, and
%% the bibliography file.
\bibliographystyle{ACM-Reference-Format}
\balance
\bibliography{my-base}

\clearpage
\section{Proof of Theorem 1}
\begin{theorem}\label{theorem1}
        Consider a non-linear layer $l$ with a function $F$. In SNNs, the output of this layer at time $t$ is denoted as $\bm{O}^{l}(t)$. Let $\bm{S}^l(T)$ be the cumulative sum of layer $l$ outputs up to time $T$, given by $\bm{S}^l(T)=\sum_{t=1}^T\bm{O}^{l}(t)$. The expected output of the SNNs at time $T$ is given by:
        \begin{equation}\label{snn_out}
            \bm{O}^{l}(T)=TF\left(\frac{\bm{S}^{l-1}(T)}{T}\right)-(T-1)F\left(\frac{\bm{S}^{l-1}(T-1)}{T-1}\right).
        \end{equation}
    \end{theorem}
\begin{proof}
    According to Section $3.2$, we denote $\bm{x}^l(t)$ as $\bm{O}^l(t)$, which has the same meaning, and we can approximate the output value of ANNs using the mean value of the output for the first T time steps in SNNs: 
    \begin{align}\label{eq1}
        \bm{a}^l_T=\Phi^l(T)=\frac{\sum_{t=1}^T\bm{O}^l(t)}{T}
    \end{align}
    where $\bm{a}^l_T$ represents the estimated values of neurons in layer $l$ at time $T$ in ANNs. It will change as the corresponding spikes in SNNs accumulate over time.
    
    Meanwhile, in the case of ANNs, $\bm{a}^l_T$ can be formulated as:
    \begin{align}\label{eq2}
        \bm{a}^l_T = F(\bm{a}^{l-1}_T).
    \end{align}
    
    %要不要加逗号
    Furthermore, we can deduce the output by subtracting the total output of the previous T and T-1 time steps from the formula \ref{eq1} and the formula \ref{eq2}.
    \begin{equation}
        \begin{aligned}
            &\bm{O}^l(T)=\sum_{t=1}^T\bm{O}^l(t)-\sum_{t=1}^{T-1}\bm{O}^l(t)
            \\&=T\bm{a}^l_T-(T-1)\bm{a}^l_{T-1}
            \\&=TF(\bm{a}^{l-1}_T)-(T-1)F(\bm{a}^{l-1}_{T-1})
            \\&=TF\left(\frac{\sum_{t=1}^TO^{l-1}(t)}{T}\right)-(T-1)F\left(\frac{\sum_{t=1}^{T-1}O^{l-1}(t)}{T-1}\right)
            \\&=TF\left(\frac{\bm{S}^{l-1}(T)}{T}\right)-(T-1)F\left(\frac{\bm{S}^{l-1}(T-1)}{T-1}\right).
        \end{aligned}
    \end{equation}
\end{proof}

\section{Proof of Theorem 2}
\begin{theorem}\label{theorem2}
    Consider a module for matrix product that receives two sets of spike inputs, denoted by $\bm{A}_{v_a}(t)$ and $\bm{B}_{v_b}(t)$. These inputs are generated by neurons $A$ and $B$, respectively, and are characterized by multiple thresholds $v_a$ and $v_b$, as described in Section 4.3. 

    We can integrate the input by $\bm{A}(t)=\sum_{v_a}v_a\bm{A}_{v_a}(t)$ and $\bm{B}(t)=\sum_{v_b}v_b\bm{B}_{v_b}(t)$. Here, $\bm{A}(t)$ and $\bm{B}(t)$ are the sum matrices weighted by multiple thresholds $v_a$ and $v_b$, respectively.

    Let $\bm{S}_{A}(T)=\sum_{t=1}^TA(t)$ and $\bm{S}_{B}(T)=\sum_{t=1}^TB(t)$ represent the cumulative sum of inputs up to time $T$. We define $\bm{S}_{K}(T)=\bm{S}_{A}(T)\bm{S}_{B}(T)$. 
    Then, the expected output at time T can be formulated as:
    
    \begin{equation}\label{matrix_exp_all}
        \begin{aligned}
            \bm{O}(T)=\frac1T\bm{S}_{K}(T)-\frac1{T-1}\bm{S}_{K}(T-1),
        \end{aligned}
    \end{equation}
    
    where $\bm{S}_{K}(T)$ can be calculated mainly using addition, as described by the following equation:
    
    \begin{equation}\label{matrix_exp}
        \bm{S}_{K}(T)=\bm{S}_{K}(T-1)+\bm{K}(T)\\
    \end{equation}
    \begin{equation}
        \begin{aligned}
            K(T)=&\sum_{v_a,v_b}v_av_b\bm{A}_{v_a}(T)\bm{B}_{v_b}(T)+\sum_{v_a}v_a\bm{A}_{v_a}(T)\bm{S}_B(T-1)
            \\&+\sum_{v_b}v_b\bm{S}_A(T-1)\bm{B}_{v_b}(T).
        \end{aligned}
    \end{equation}
    
\end{theorem}
\begin{proof}
    Since we approximate the value of ANNs using the mean value for the first T times in SNNs,
    let the expected input matrices $\bm{A}_T$, $\bm{B}_T$, and $\bm{O}_T=\bm{A}_T\bm{B}_T$ in ANNs be calculated based on the input spikes during the first $T$ time steps in SNNs, denoted as:
    \begin{align}
        &\bm{A}_T=\frac{\sum_{t=1}^T\bm{A}(t)}{T}
        \\&\bm{B}_T=\frac{\sum_{t=1}^T\bm{B}(t)}{T}
        \\&\bm{O}_T=\frac{\sum_{t=1}^T\bm{O}(t)}{T}
    \end{align}
    So, the expected output matrix $\bm{O}(T)$ at time $T$ can be calculated by:
    \begin{equation}\label{matrix_exp2}
        \begin{aligned}
            &\bm{O}(T)=\sum_{t=1}^T\bm{O}(t)-\sum_{i=t}^{T-1}\bm{O}(t)
            \\&=T\bm{O}_T-(T-1)\bm{O}_{T-1}
            \\&=T\bm{A}_T\bm{B}_T-(T-1)\bm{A}_{T-1}\bm{B}_{T-1}
            \\&=T\frac{\sum_{t=1}^T\bm{A}(t)}T\frac{\sum_{t=1}^T\bm{B}(t)}T
            \\&\quad -(T-1)\frac{\sum_{t=1}^{T-1}\bm{A}(t)}{T-1}\frac{\sum_{t=1}^{T-1}\bm{B}(t)}{T-1}
            \\&=\frac1T\sum_{t=1}^T\bm{A}(t)\sum_{t=1}^T\bm{B}(t)-\frac1{(T-1)}\sum_{t=1}^{T-1}\bm{A}(t)\sum_{t=1}^{T-1}\bm{B}(t)
            \\&=\frac1T\bm{S}_A(T)\bm{S}_B(T)-\frac1{T-1}\bm{S}_A(T-1)\bm{S}_B(T-1)
            \\&=\frac1T\bm{S}_K(T)-\frac1{T-1}\bm{S}_K(T-1)
        \end{aligned}
    \end{equation}
    And $\bm{S}_{K}(T)$ can be calculated by:
    \begin{equation}\label{matrix_sum_out}
        \begin{aligned}
        &\bm{S}_{K}(T)=\bm{S}_{A}(T)\bm{S}_{B}(T)
        % \\&=(\sum_{t=1}^T\bm{A}(t))(\sum_{t=1}^T\bm{B}(t))
        \\&=(\bm{S}_{A}(T-1)+\bm{A}(T))(\bm{S}_{B}(T-1)+\bm{B}(T))
        \\&=\bm{S}_{A}(T-1)\bm{S}_{B}(T-1)+\bm{A}(T)\bm{B}(T)
        \\&\quad+\bm{A}(T)\bm{S}_{B}(T-1)+\bm{S}_{A}(T-1)\bm{B}(T)
        \\&=\bm{S}_{K}(T-1)+\sum_{v_a,v_b}v_av_b\bm{A}_{v_a}(T)\bm{B}_{v_b}(T)
        \\&\quad+\sum_{v_a}v_a\bm{A}_{v_a}(T)\bm{S}_B(T-1)+\sum_{v_b}v_b\bm{S}_A(T-1)\bm{B}_{v_b}(T)
        \\&=\bm{S}_{K}(T-1)+\bm{K}(T).
        \end{aligned}
    \end{equation}
    Assuming the dimension of $\bm{S}_{K}(T)$, $\bm{S}_{A}(T)$ and  $\bm{S}_{B}(T)$ are $n\times m$, $n\times p$ and $p\times m$, respectively. And suppose the firing rate of $A(T)$ and $B(T)$ are $\eta_1$ and $\eta_2$.
    
    In order to determine the number of different operations required to update $\bm{S}_{K}(T)$, we conduct a brief analysis: 
    Multiplications occur when the threshold is multiplied by the results of various matrix multiplications; 
    Additions occur during the calculation of individual matrix multiplications, as well as the accumulation of the results of the four parts.
    
    As each position of the input matrix has only one effective threshold at each time, it restricts the total number of input spikes, thus limiting the total number of operations.

    The maximum addition operation number is 
    \begin{align}
        ACs_\text{SNN}^{max}=\eta_1\eta_2npm+\eta_1npm+\eta_2npm+3nm
    \end{align}
    where $\eta_1\eta_2npm$, $\eta_1npm$ and $\eta_2npm$ are the maximum addition operations in calculating $\sum_{v_a,v_b}v_av_b\bm{A}_{v_a}(T)\bm{B}_{v_b}(T)$
        , $\sum_{v_a}v_a\bm{A}_{v_a}(T)\bm{S}_B(T-1)$ and $\sum_{v_b}v_b\bm{S}_A(T-1)\bm{B}_{v_b}(T)$, respectively. $3nm$ is the maximum operation in accumulating four parts in Equation \eqref{matrix_exp}.
        
    The maximum multiplication operation number is 
    \begin{align}
        MACs_\text{SNN}^{max}=min(\eta_1,\eta_2)nm+\eta_1nm+\eta_2nm
    \end{align}
    where $min(\eta_1,\eta_2)nm$, $\eta_1nm$ and $\eta_2nm$ are the maximum multiplication operations in calculating $\sum_{v_a,v_b}v_av_b\bm{A}_{v_a}(T)\bm{B}_{v_b}(T)$, $\sum_{v_a}v_a\bm{A}_{v_a}(T)\bm{S}_B(T-1)$ and $\sum_{v_b}v_b\bm{S}_A(T-1)\bm{B}_{v_b}(T)$, respectively.
    
    It can be seen that $ACs_\text{SNN}^{max} \gg MACs_\text{SNN}^{max}$, so $S_K(T)$ can be calculated mainly using addition.
    \end{proof}
\section{Experiment Details}
\subsection{Datasets}
\quad \textbf{CIFAR-10}. The CIFAR-10 dataset \cite{krizhevsky2009learning} consists of 60000 32 × 32 images in
10 classes. There are 50000 training images and 10000 test images.

\textbf{CIFAR-100}. The CIFAR-100 dataset \cite{krizhevsky2009learning} consists of 60000 32 × 32 images in
100 classes. There are 50000 training images and 10000 test images.

\textbf{ImageNet1k}. We use the ILSVRC 2012 dataset \cite{russakovsky2015imagenet}, which consists of 1,281,167
training images and 50000 testing images.
\subsection{Data Preprocessing}
To process our image data, we followed a series of steps. First, we resized the image to the desired size and then cropped it to match the input size. After that, we converted the image into a PyTorch tensor. 
Next, we normalized the pixel values using the provided mean and standard deviation values. The mean and standard deviation values were specified as (0.48145466, 0.4578275, 0.40821073) and (0.26862954, 0.26130258, 0.27577711). Finally, we normalized the pixel values of the three-channel images based on the provided mean and standard deviation.

    \begin{table*}[htb]
        \centering
        \caption{Accuracy and energy consumption ratio of ECMT(Ours) on CIFAR10 dataset}
        \label{accandenergycifar10}
        %\small{
        \begin{tabular}{ccccccccc}
        \toprule
        \multirow{2}{*}{Arch.}&\multirow{2}{*}{Accuracy/Energy}&\multirow{2}{*}{Original (ANN)}&\multicolumn{6}{c}{Ours (SNN)}\\
        \cmidrule(lr){4-9}
        &&&T=1&T=2&T=4&T=6&T=8&T=10\\
        \midrule
        \multirow{2}{*}{ViT-S/16}&Acc.~(\%)&98.33&8.53&31.32&93.82&97.37&98.01&98.21\\
        &Energy ratio&1&0.06&0.15&0.37&0.60&0.82&1.03\\
        \midrule
        \multirow{2}{*}{ViT-B/16}&Acc.~(\%)&98.75&9.17&32.25&95.17&98.24&98.55&98.69\\
        &Energy ratio&1&0.04&0.12&0.30&0.48&0.66&0.83\\
        \midrule
        \multirow{2}{*}{ViT-L/16}&Acc.~(\%)&99.07&10.55&95.14&98.89&99.1&99.03&99.08\\
        &Energy ratio&1&0.03&0.11&0.27&0.42&0.57&0.72\\
        \bottomrule
        \end{tabular}
        %}
    \end{table*}
    
    \begin{table*}[htb]
        \centering
        \caption{Accuracy and energy consumption ratio of ECMT(Ours) on CIFAR100 dataset}
        \label{accandenergycifar100}
        %\small{
        \begin{tabular}{ccccccccc}
        \toprule
        \multirow{2}{*}{Arch.}&\multirow{2}{*}{Accuracy/Energy}&\multirow{2}{*}{Original (ANN)}&\multicolumn{6}{c}{Ours (SNN)}\\
        \cmidrule(lr){4-9}
        &&&T=1&T=2&T=4&T=6&T=8&T=10\\
        \midrule
        \multirow{2}{*}{ViT-S/16}&Acc.~(\%)&89.28&0.95&4.9&69.49&84.75&87.83&88.93\\
        &Energy ratio&1&0.06&0.16&0.38&0.61&0.84&1.07\\
        \midrule
        \multirow{2}{*}{ViT-B/16}&Acc.~(\%)&92.26&0.87&17.07&82.86&90.22&91.5&91.91\\
        &Energy ratio&1&0.04&0.12&0.30&0.48&0.66&0.84\\
        \midrule
        \multirow{2}{*}{ViT-L/16}&Acc.~(\%)&93.84&1.61&69.08&91.82&93.04&93.34&93.56\\
        &Energy ratio&1&0.04&0.12&0.27&0.43&0.58&0.73\\
        \bottomrule
        \end{tabular}
        %}
    \end{table*}
    
    \begin{table*}[htb]
        \centering
        %\small{
        \caption{Comparison between the proposed method and previous works on CIFAR10 dataset}
        \begin{tabular}{cccccc}
        \toprule
        Method & Type & Arch. & Param.~(M) & T & Accuracy (\%) \\
        \midrule
        Spikingformer\cite{zhou2023spikingformer}&Direct Training&Spikingformer-4-384-400E&9.32&4&95.81\\
        Spike-driven Transformer\cite{yao2024spike}&Direct Training&Spikingformer-4-384-400E&9.32&4&95.6\\
        RMP\cite{han2020rmp}&CNN-to-SNN&VGG-16&138&64(2048)&90.35(93.63)\\
        SNM\cite{wang2022signed}&CNN-to-SNN&VGG-16&138&32(128)&93.43(94.07)\\
        TS\cite{deng2021optimal}&CNN-to-SNN&VGG-16&138&16(32)&92.29(92.29)\\
        QFFS\cite{li2022quantization}&CNN-to-SNN&VGG-16&138&4&92.64\\
        \cmidrule(l){3-6}
        \multirow{2}{*}{QCFS\cite{bu2022optimal}}&\multirow{2}{*}{CNN-to-SNN}&ResNet-18&11.8&8(64)&94.82(96.06)\\
        &&VGG-16&138&8(64)&94.95(95.55)\\
        \cmidrule(l){3-6}
        \multirow{2}{*}{SRP\cite{hao2023reducing}}&\multirow{2}{*}{CNN-to-SNN}&ResNet-18&11.8&4(16)&95.25(95.55)\\
        &&VGG-16&138&4(16)&95.32(95.42)\\
        \cmidrule(l){3-6}
        MST\cite{wang2023masked}&Transformer-to-SNN&Swin-T(BN)&27.6&64(256)&96.32(97.27)\\
        \cmidrule(l){3-6}
        STA\cite{jiang2024spatiotemporal}&Transformer-to-SNN&ViT-B/32&86&32(256)&95.49(95.82)\\
        \cmidrule(l){3-6}
        \multirow{3}{*}{{\bf ECMT(Ours)}}&\multirow{3}{*}{Transformer-to-SNN}&ViT-S/16&22&6(8)&97.37(98.01)\\
        &&ViT-B/16&86&6(8)&98.24(98.55)\\
        &&ViT-L/16&307&6(8)&99.1(99.03)\\
        \bottomrule
        \end{tabular}
        %}
        \label{cmpcifar10}
    \end{table*}
    
    \begin{table*}[htb]
        \centering
        %\small{
        \caption{Comparison between the proposed method and previous works on CIFAR100 dataset}
        \begin{tabular}{cccccc}
        \toprule
        Method & Type & Arch. & Param.~(M) & T & Accuracy (\%) \\
        \midrule
        Spikingformer\cite{zhou2023spikingformer}&Direct Training&Spikingformer-4-384-400E&9.32&4&79.21\\
        Spike-driven Transformer\cite{yao2024spike}&Direct Training&Spikingformer-4-384-400E&9.32&4&78.4\\
        RMP\cite{han2020rmp}&CNN-to-SNN&VGG-16&138&128(2048)&63.76(70.93)\\
        SNM\cite{wang2022signed}&CNN-to-SNN&VGG-16&138&32(128)&71.8(73.95)\\
        TS\cite{deng2021optimal}&CNN-to-SNN&VGG-16&138&16(64)&63.73(69.27)\\
        \cmidrule(l){3-6}
        \multirow{2}{*}{QCFS\cite{bu2022optimal}}&\multirow{2}{*}{CNN-to-SNN}&ResNet-18&11.8&8(64)&78.48(79.54)\\
        &&VGG-16&138&8(64)&73.96(77.10)\\
        \cmidrule(l){3-6}
        \multirow{2}{*}{SRP\cite{hao2023reducing}}&\multirow{2}{*}{CNN-to-SNN}&ResNet-20&0.27&4(32)&59.34(65.50)\\
        &&VGG-16&138&4(32)&75.42(76.45)\\
        \cmidrule(l){3-6}
        MST\cite{wang2023masked}&Transformer-to-SNN&Swin-T(BN)&27.6&64(256)&85.4(86.91)\\
        \cmidrule(l){3-6}
        STA\cite{jiang2024spatiotemporal}&Transformer-to-SNN&ViT-B/32&86&32(256)&84.15(85.98)\\
        \cmidrule(l){3-6}
        \multirow{3}{*}{{\bf ECMT(Ours)}}&\multirow{3}{*}{Transformer-to-SNN}&ViT-S/16&22&6(8)&84.75(87.83)\\
        &&ViT-B/16&86&6(8)&90.22(91.5)\\
        &&ViT-L/16&307&6(8)&93.04(93.34)\\
        \bottomrule
        \end{tabular}
        %}
        \label{cmpcifar100}
        
    \end{table*}
    %     \begin{table*}[htb]
    %     \centering
    %     %\small{
    %     \caption{Proposed method on CIFAR10 dataset}
    %     \begin{tabular}{cccccc}
    %     \toprule
    %     Method & Type & Arch. & Param.~(M) & T & Accuracy (\%) \\
    %     \midrule
    %     \multirow{3}{*}{{\bf ECMT(Ours)}}&\multirow{3}{*}{Transformer-to-SNN}&ViT-S/16&22&6(8)&97.37(98.01)\\
    %     &&ViT-B/16&86&6(8)&98.24(98.55)\\
    %     &&ViT-L/16&307&6(8)&99.1(99.03)\\
    %     \bottomrule
    %     \end{tabular}
    %     %}
    %     \label{cmpcifar10}
    % \end{table*}
    % \begin{table*}[htb]
    %     \centering
    %     %\small{
    %     \caption{Proposed method on CIFAR100 dataset}
    %     \begin{tabular}{cccccc}
    %     \toprule
    %     Method & Type & Arch. & Param.~(M) & T & Accuracy (\%) \\
    %     \midrule
        
    %     \multirow{3}{*}{{\bf ECMT(Ours)}}&\multirow{3}{*}{Transformer-to-SNN}&ViT-S/16&22&6(8)&84.75(87.83)\\
    %     &&ViT-B/16&86&6(8)&90.22(91.5)\\
    %     &&ViT-L/16&307&6(8)&93.04(93.34)\\
    %     \bottomrule
    %     \end{tabular}
    %     %}
    %     \label{cmpcifar100}
    % \end{table*}
    \subsection{Experimental Setup}
    The conversion in this paper is based on pre-trained Vision Transformer including the $\text{ViT-S/16}$, $\text{ViT-B/16}$, $\text{ViT-L/16}$ with 224 resolution \cite{vaswani2017attention}, and the $\text{EVA}$ model $\text{eva\_g\_patch14}$ in \cite{fang2023eva}. 
    
    For all Multi-Threshold Neurons, we set $n$ to 8 for $\text{ViT-S/16}$, $\text{ViT-B/16}$, $\text{ViT-L/16}$ and 6 for $\text{EVA}$. 
    We set threshold percent $p$ to 99 to get thresholds for each neuron. In particular, due to huge differences in GELU and softmax layers' output values, we configure the positive and negative base thresholds to 0.5 and 0.08, respectively, for neurons following the GELU module in $ ViT$ models, and to 0.0125 for neurons following the softmax module to prevent too few spikes. 
    
    Besides, the precision of the network is highly sensitive to the precision of the classification layer, as mentioned in \cite{li2022quantization}. Since the classification layer has minimal energy consumption during runtime, we retained analog input in the classification layer.

\section{Additional Experimental Details}

    \subsection{Detailed results on other datasets}
    Tables \ref{accandenergycifar10} and \ref{accandenergycifar100} present a comparison of the accuracy and energy consumption of different neural network architectures - ANNs and SNNs - on CIFAR10 and CIFAR100 datasets.
    
    Table \ref{accandenergycifar10} compares the accuracy of ANN and SNN architectures for the CIFAR10 dataset across three model scales: $\text{ViT-S/16}$, $\text{ViT-B/16}$, and $\text{ViT-L/16}$. It can be seen that the SNN model can reach a comparable accuracy while significantly reducing the consumption. For example, when the SNN model is run for 6 time steps, models such as $\text{ViT-S/16}$, $\text{ViT-B/16}$, and $\text{ViT-L/16}$ achieve accuracy levels of 97.37\%, 98.24\%, and 99.1\%, respectively. The remarkable fact is that they only consume 0.6, 0.48, and 0.4 energy, respectively when compared to the original ANN (Artificial Neural Network) models.
    
    Table \ref{accandenergycifar100} presents a similar comparison for the more complex CIFAR100 dataset. For instance, at 6 timesteps, $\text{ViT-S/16}$, $\text{ViT-B/16}$, and $\text{ViT-L/16}$ achieve accuracies of 84.75\%, 90.22\%, and 93.04\%, respectively, while using only 0.61, 0.48, and 0.43 energy compared to original ANN models. It shows the potential of our method to reduce energy consumption while maintaining accuracy. The results demonstrate our method's potential to reduce energy consumption while maintaining accuracy.
    
    \subsection{Comparison with the State-of-the-art on CIFAR10 and CIFAR100 datasets}
    We compare the experimental results using the $\text{ViT-S/16}$, $\text{ViT-B/16}$, $\text{ViT-L/16}$ model on the CIFAR10 and CIFAR100 datasets with previous state-of-the-art methods, as shown in Table \ref{cmpcifar10} and \ref{cmpcifar100}.

    In the evaluation of the CIFAR10 dataset, the ECMT model achieved an impressive accuracy rate of 97.37\%, 98.24\%, and 99.1\% respectively, using the architecture of $\text{ViT-S/16}$, $\text{ViT-B/16}$, $\text{ViT-L/16}$ over just six timesteps. This level of precision is highly competitive, especially compared to similarly-sized models. In evaluating the CIFAR100 dataset, considered more complex, the ECMT method again displays its strength. The results demonstrate that the ECMT method achieves a similar high accuracy.
    
    The ECMT model uses the Transformer-to-SNN approach and has performed exceptionally well on the CIFAR10 and CIFAR100 datasets. Its $\text{ViT-B/16}$ variant stands out by achieving high accuracy with a moderate number of parameters, indicating the potential of SNNs in achieving state-of-the-art results with a significant reduction in computational resources. This balance of efficiency and accuracy makes the ECMT a promising model for energy-efficient and fast processing tasks.
    \balance

% \clearpage
% %%
% %% The acknowledgments section is defined using the "acks" environment
% %% (and NOT an unnumbered section). This ensures the proper
% %% identification of the section in the article metadata, and the
% %% consistent spelling of the heading.

% \begin{acks}
% This work was supported by the National Natural Science Foundation of China (62176003, 62088102) and the Beijing Nova Program (20230484362).
% \end{acks}

%%
%% The next two lines define the bibliography style to be used, and
%% the bibliography file.
% \bibliographystyle{ACM-Reference-Format}
% \balance
% \bibliography{my-base}

% %%
% %% If your work has an appendix, this is the place to put it.

% \end{document}

\end{document}